\documentclass{colt2015} 

\title[Fast Mixing for Discrete Point Processes]{Fast Mixing for Discrete Point Processes\thanks{This is the full version of a paper in the 28th Annual Conference on Learning Theory (COLT), 2015.}}

\usepackage{times}
\DeclareFontFamily{U}{mathx}{\hyphenchar\font45}
\DeclareFontShape{U}{mathx}{m}{n}{
  <5> <6> <7> <8> <9> <10>
  <10.95> <12> <14.4> <17.28> <20.74> <24.88>
  mathx10
}{}
\DeclareSymbolFont{mathx}{U}{mathx}{m}{n}
\DeclareMathSymbol{\bigtimes}{1}{mathx}{"91}
\usepackage{algorithm}

 \coltauthor{\Name{Patrick Rebeschini} \Email{patrick.rebeschini@yale.edu}\\
 \Name{Amin Karbasi} \Email{amin.karbasi@yale.edu}\\
 \addr Yale Institute for Network Science, 17 Hillhouse Avenue, Yale University, New Haven, CT, 06511, USA}

\begin{document}

\maketitle
\thispagestyle{empty}

\begin{abstract}
We investigate the systematic mechanism for designing fast mixing Markov chain Monte Carlo algorithms to sample from discrete point processes under the Dobrushin uniqueness condition for Gibbs measures. Discrete point processes are defined as probability distributions $\mu(S)\propto \exp(\beta f(S))$ over all subsets $S\in 2^V$ of a finite set $V$ through a bounded set function $f:2^V\rightarrow \mathbb{R}$ and a parameter $\beta>0$. A subclass of discrete point processes characterized by submodular functions (which include log-submodular distributions, submodular point processes, and determinantal point processes) has recently gained a lot of interest in machine learning and shown to be effective for modeling diversity and coverage.
We show that if the set function (not necessarily submodular) displays a natural notion of decay of correlation, then, for $\beta$ small enough, it is possible to design fast mixing Markov chain Monte Carlo methods that yield error bounds on marginal approximations that do not depend on the size of the set $V$. The sufficient conditions that we derive involve a control on the (discrete) Hessian of set functions, a quantity that has not been previously considered in the literature. We specialize our results for submodular functions, and we discuss canonical examples where the Hessian can be easily controlled.
\end{abstract}

\begin{keywords}
Discrete point processes, MCMC, fast mixing, submodular functions, decay of correlation, Hessian of set functions
\end{keywords}

\section{Introduction}

Probabilistic modeling and inference techniques have become essential tools for analyzing
data and making predictions in a variety of real-world settings.
Graphical models \citep{wainwright2008graphical} have provided an appealing framework to expressed dependencies among variables through a graph structure. A broad class of such models that have been widely used in machine learning is represented by Markov random fields, where the probability distribution of a collection of $n$ random variables $X:=(X^1,\ldots,X^n)$ is defined as a product of non-negative potentials $\phi_C$ over maximal cliques $C\subseteq G$ of an undirected graph $G=(V,E)$, i.e., $\mu(x):=\mathbf{P}(X=x) = 1/Z\prod_{C} \phi_C (x^C)$, where we adopted the notation $x^C:=\{x^i\}_{i\in C}$. Here, $Z$ is the normalization factor, often called the \textit{partition function}, and it is known to be hard to compute exactly \citep{JS93}, or even to approximate \citep{goldberg2007complexity}. Perhaps the most prominent example of Markov networks, with many applications in machine learning, is the pairwise Markov random field, also called Ising model,  where cliques are defined on edges between pairs of variables.

These examples can be seen as instances of a general class of probabilistic models that we refer to as \textit{discrete point processes}, which are defined as
\begin{equation}\label{eq:dis}
	\mu(S) = \frac{1}{Z} \exp(\beta f(S))
	\qquad \text{for } S\subseteq V,
\end{equation}
where $f:2^V\rightarrow \mathbb{R}$ is a bounded set function, $Z := \sum_{S\subseteq V} \exp(\beta f(S))$ is the partition function, and $\beta$ is a strictly positive real constant that parametrizes the distribution. Discrete point processes have been widely studied in mathematics and statistical physics \citep{daley2007introduction}[ch. 5], where they have been traditionally used to model particle processes and neural spiking activity, among many others. Note that distributions over subsets of $V$ are isomorphic to distributions of $n:=|V|$ binary random variables $X^1, \dots, X^n\in\{0,1\}$. 

A subclass of discrete point processes, referred to as log-submodular (log-supermodular) distributions, has recently been investigated in \cite{DK14}, where the set function $f$ is taken to be \textit{submodular} (respectively, \textit{supermodular}), i.e., characterized by the property that the difference in the value of the function when an element is added to a set, the so-called \emph{marginal gain}, decreases (increases) as the cardinality of the set increases. Throughout this paper, the discrete derivative, or marginal gain, of $f$ is defined as $\Delta_i f(S):=f(S\cup\{i\})-f(S)$. The function $f$ is submodular (supermodular) if for any $S\subseteq S'\subseteq V$ and any $i\in V\setminus S'$, it holds $\Delta_i f(S) \ge \Delta_i f(S')$ ($\Delta_i f(S) \le \Delta_i f(S')$).
Under some regularity conditions (discussed in \cite{DK14}), pairwise Markov random fields are also a special case of log-submodular point processes, as are determinantal point processes \citep{kulesza2012determinantal}, \citep{hough2006determinantal}.
Here, the submodular function is defined as $f(S) = \log\det(L_S)$ where $L\in \mathbb{R}^{n\times n}$ is a positive definite matrix and $L_S$ is the square submatrix of $L$ that is indexed by $S$. Determinantal point processes have been extensively used in physics and machine learning to model negative correlations, giving rise, for instance, to diverse sets of items in recommendations.
Another related subclass of discrete point processes, called submodular (supermodular) point processes, has been recently proposed in \cite{iyer2015-spps}, where $\mu(S)\propto f(S)$ with $f$ a non-negative submodular (supermodular) function.

The diminishing return property that characterizes submodularity --- which makes submodular functions suitable for applications in several fields, ranging from economics to machine learning --- has been extensively investigated in the domain of optimization (see \cite{krause14survey} for instance), but its role has yet to be established in the realm of probabilistic inference. \cite{iyer2015-spps} showed that, in general, computing the partition function in log-submodular distributions and submodular point processes requires exponential complexity. \cite{DK14} and \cite{iyer2015-spps} resort to variational approaches to approximate the partition function. 
In \cite{DK14} the authors provide upper and lower bounds based on sub- and super- gradients \citep{iyer2013fast}, showing that the log partition function $\log(Z)$ can be approximated within $O(n)$. In their setting, however, this implies that the error bounds that can be derived from their theory to approximate marginals of the type $\mu(\{S\subseteq V: S\ni i\})$, for a given $i\in V$, deteriorate exponentially with the size of the model $n=|V|$, as can be deduced from their experimental results.
In addition, the bounds that they consider depend on the \emph{curvature} $c(f)$ of the submodular function $f$, a quantity between $0$ and $1$ that characterizes the deviation from modularity ($f$ is said to be \emph{modular} if $\Delta_i f(S) = \Delta_i f(S')$ for all $S,S'\subseteq V$, $i\not\in S,S'$, and in this case $c(f)=0$). However, there are trivial examples with very little interactions between random variables (Section \ref{sec:Modular functions}) for which the inference problem has a straightforward solution but $c(f)=1$, so that the corresponding (lower/upped) bounds on the partition function in \cite{DK14} are unbounded and no useful inference can be deduced. In contrast to Markov random fields, whose partition function is typically intractable and hard to approximate, determinantal point processes admit an exact algorithm for marginalization, albeit in time cubic in the size $n$. In order to avoid this cost, recently \cite{K13} considered a Markov chain Monte Carlo (MCMC) algorithm to sample from determinant point processes. The author claims that this algorithm is fast mixing (i.e., the Markov chain gets arbitrarily close to equilibrium after a small, $O(n\log n)$, number of steps) with no additional assumptions on the model. However, the proofs of the results in this paper are wrong as they rely on ill-defined couplings between Markov chains.

In this paper we address the inference problem of approximating marginals of the type $\mu(\{S\subseteq V: S\ni i\})$ for discrete point processes defined in Eq.~\eqref{eq:dis} through MCMC methods. In Section \ref{sec:Fast mixing MCMC algorithms} we define the general class of local-update algorithms that we consider, and we present the main theoretical result for fast mixing, i.e., Theorem \ref{thm:fastMCMC}, which relies on the Dobrushin uniqueness condition for Gibbs measures \citep{Dob70}, \citep{Geo11}. In Section \ref{sec:Fast mixing local MCMC algorithms for discrete point processes} we analyze two specific algorithms within this class, Gibbs sampling and Metropolis-Hastings, and we show that if the set function $f$ satisfies some natural notion of \textit{decay of correlation}, then these algorithms are fast mixing and yield size-free error bounds on marginal approximations.
The decay of correlation property that we exploit concerns the decay of the absolute value of the \emph{difference} of marginal gains evaluated at sets differing only by a single element $j$, as a function of the element $i$ being added (in fact, here the role of $i$ and $j$ is symmetric). More precisely, this property is related to the second order derivatives $\Delta_j\Delta_i f(S) = \Delta_i f(S \cup \{j\}) - \Delta_i f(S) \equiv \Delta_i\Delta_j f(S)$, i.e., to the (discrete) \emph{Hessian} of the function $f$. If we define $M_{ij} := \max_{S\subseteq V: S\not\ni i,j} |\Delta_j\Delta_i f(S)|$ for each $i\neq j$, and $M_{ii}:=0$ for each $i$, we show  that the Gibbs sampler is fast mixing if the following condition holds
$$
	\alpha(\beta) \beta \| M \|_\infty \le \gamma < 1,
$$ 
where $\| M \|_\infty:=\max_{i\in V}\sum_{j\in V} M_{ij}$, $\alpha(\beta):= \max_{i\in V} \max_{S\subseteq V\setminus\{i\}} e^{-\beta \Delta_if(S)}$, and $\gamma\ge 0$ is a quantity that does not depend on $n$. If the set function is submodular, i.e., the distribution is log-submodular, then we can simplify this condition (Lemma \ref{lem:Gibbs sampler fast mixing}). To the best of our knowledge, these results are the first to emphasize the importance of the Hessian of set functions, a quantity that has not been previously investigated even in the optimization domain. 
Finally, in Section \ref{sec:applications} we specialize our results for a  number of canonical examples of submodular functions (facility location, cut function, log determinant functions leading to determinantal point processes, and decomposable functions). These examples attest that our general criterion, which a priori involves a combinatorial optimization to compute each term $M_{ij}$, can often be reduced to a simple-to-check condition. Proofs are given in Appendix \ref{sec:Proofs of MCMC convergence results} (theory) and Appendix \ref{sec:Proofs of Applications results} (applications). As a final remark, we should highlight that submodularity (supermodularity), which is equivalent to $\Delta_j\Delta_i f(S)\leq 0$ ($\Delta_j\Delta_i f(S)\ge 0$) for any $i,j\in V$, $i\neq j$, $S\subseteq V$, $S\not\ni i,j$, is not sufficient to guarantee fast mixing, as displayed by the different convergence behaviors of the Glauber dynamics for Ising models with respect to different values of the inverse temperature $\beta$. See \cite{mossel2013} and references therein.\\

\textbf{Notation.}
In this paper we adopt the usual vector/matrix notation for distributions, kernels, and functions defined on finite sets. Given two finite sets $\mathbb{X}$ and $\mathbb{Y}$ with respective cardinality $|\mathbb{X}|$ and $|\mathbb{Y}|$, we interpret a probability distribution $\rho$ on $\mathbb{X}$ as $|\mathbb{X}|$-dimensional row vector, a kernel $T$ from $\mathbb{X}$ to $\mathbb{Y}$ as a matrix in $[0,1]^{|\mathbb{X}|\times |\mathbb{Y}|}$, and a function $h:\mathbb{Y}\rightarrow \mathbb{R}$ as a $|\mathbb{Y}|$-dimensional column vector. Hence, we write $\rho T h := \sum_{x\in \mathbb{X},y\in \mathbb{Y}}\rho(x) T(x,y) h(y)$. For each $x\in\mathbb{X}$, we write $T_x$ to indicate the probability distribution $T_x:y\in\mathbb{Y} \rightarrow T_x(y):=T(x,y)$, and for $m \ge 0$ we write $T^m$ to indicate the $m$-th power of the matrix $T$. Given $A\subseteq \mathbb{Y}$, we define the indicator function $1_A$ as $1_A(y):=1$ if $y\in A$, $1_A(y):=0$ if $y\not\in A$. Clearly, $\rho (A)=\rho 1_A$. For $y\in\mathbb{Y}$, we will also use the notation $1_y$ to mean $1_{\{y\}}$.
If $V$ is a finite set, given $x=(x^i)_{i\in V}$, we write $x^S:=(x^i)_{i\in S},$ for $S\subseteq V$. If $x,y\in\mathbb{R}$, we use $x\land y := \min\{x,y\}$ and $x\vee y := \max\{x,y\}$.

\section{Fast mixing MCMC algorithms for discrete point processes}\label{sec:Fast mixing MCMC algorithms}

Throughout this paper, let $V$ be a finite set with cardinality $n:=|V|$. Let $\mathbb{S}_0:=\{0,1\}$ and define $\mathbb{S}:=\bigtimes_{i\in V} \mathbb{S}_0= \{0,1\}^n$. Given $\beta >0$, we consider the following probability distribution $\mu$ on $\mathbb{S}$
\begin{align}
	\mu(x) :=
	\frac{e^{\beta f(x)}
	}
	{\sum_{x\in \mathbb{S}} e^{\beta f(x)}},
	\label{basic model}
\end{align}
for $x\in\mathbb{S}$, where $f:\mathbb{S}\rightarrow\mathbb{R}$ is a given bounded function. In the following we will also consider the isomorphic description of the function $f$ given by the set function $f^\star:2^V\rightarrow\mathbb{R}$ defined as follows, for each $S\subseteq V$:
$
	f^\star(S) := f(x(S)),
$
where $x(S)\in\mathbb{S}$ is defined as $x(S)^i=1 \text{ if } i\in S, x(S)^i=0 \text{ if } i\not\in S$. With an overload of notation, henceforth we refer to $f^\star$ as $f$, leaving to the context the determination of what is meant. In this paper we address the problem of probabilistic inference for discrete point processes defined in \eqref{basic model}, that is, the problem of computing guaranteed approximations to marginals probabilities of the type
\begin{align}
	\mu(\{x\in\mathbb{S}:x^i= 1 \ \forall i\in S\})
	=
	\sum_{x\in \mathbb{S}: x^i=1, i\in S} \mu(x),
	\label{goal:marginals}
\end{align}
for a given $S\subseteq V$.
Note that, in general, computing \eqref{goal:marginals} exactly is hard, as computing the normalization function in \eqref{basic model} is \#P-complete, see \cite{JS93}. Hence, we need to resort to approximation schemes.
Our goal is to investigate the properties of the function $f$ that make it possible to design time-uniform Markov chains that quickly converge to $\mu$. More specifically, we want to design a transition kernel $T$ on $\mathbb{S}$ such that $\mu T = \mu$, and so that, for any distribution $\rho$ on $\mathbb{S}$ and any function $h:\mathbb{S}\rightarrow\mathbb{R}$, $\rho T^mh$ converges exponentially fast to $\mu h$, as $m$ increases. As $|\mathbb{S}|=2^n$, a general transition kernel $T$ on $\mathbb{S}$ is a matrix with at most $2^n(2^n-1)$ degrees of freedom (recall that for each $x\in\mathbb{S}$ it has to hold $\sum_{z\in\mathbb{S}} T(x,z)=1$). In particular, for each $x\in\mathbb{S}$, $T_x$ is a distribution with $2^{n-1}$ degrees of freedom, which can be difficult to sample from if $n$ is large. To avoid this exponential burden with the cardinality of $\mathbb{S}$, we restrict our attention on Markov chains that are defined as combinations of \emph{local-update} probability kernels $K^{[i]}$ from $\mathbb{S}$ to $\mathbb{S}_0$ that we assume we can easily sample from (in particular, this implies that each $K^{[i]}$ should not depend on the normalization function in \eqref{basic model}). 
Hence, for each $i\in V$ define the transition kernel $T^{[i]}$ on $\mathbb{S}$ as
\begin{align}
	T^{[i]}(x,z):=K^{[i]}(x,z^i) 1_{x^{V\setminus \{i\}}}(z^{V\setminus \{i\}}),
	\label{transition kernel}
\end{align}
where, for each $i\in V$, $K^{[i]}$ is a probability kernel from $\mathbb{S}$ to $\mathbb{S}_0$ so that $T^{[i]}$ leaves $\mu$ invariant, namely, $\mu T^{[i]} =\mu$. 
Label the elements of $V$ as $V=\{i_1,\ldots,i_n\}$. We consider the two chains:
\begin{align}
	T_{s} &:= T^{[i_1]} \cdots T^{[i_n]},
	&\text{(systematic scan)}
	\label{MCMC-systematic}\\
	T_{r} &:= \left(\frac{1}{n} \sum_{i\in V} T^{[i]}\right)^n.
	&\text{(random scan)}
	\label{MCMC-random}
\end{align}
Markov chains described by these type of transition kernels are usually referred to as \emph{systematic scan} \eqref{MCMC-systematic} and \emph{random scan} \eqref{MCMC-random} Markov chain Monte Carlo (MCMC)\footnote{Typically in the literature (see \cite{dyer2009}, for instance) the random scan MCMC sampler is defined as $S:=\frac{1}{n} \sum_{i\in V} T^{[i]}$, instead of $T_r=S^n$ as in \eqref{MCMC-random}. Our choice in the present context is motivated by the fact that we want to compare random and systematic scan. Note, in fact, that a single application of the kernel $T_s$ in \eqref{MCMC-systematic} involves updating all coordinated $i_1,\ldots,i_n$, while a single application of $S$ involves updating (uniformly at random) only one coordinate $i\in V$. This is why the right scale to make the comparison involves $n$ iterations of $S$.}. Respectively, they give rise to Algorithm \ref{alg:generalMCMC-systematic} and Algorithm \ref{alg:generalMCMC-random}, when the chains are run for $m\ge 1$ \emph{sweeps} with initial distribution $\rho$.

\begin{algorithm}[H]
	Sample $X\in\mathbb{S}$ from the distribution $\rho$;\\
	\For{$k=1,\ldots,m$}{
	\For{$j=1,\ldots,n$}{
		Draw $Z^{i_j}$ from the distribution $K^{[i_j]}_x$;\\
		Set $X^{i_j} \leftarrow Z^{i_j}$;
	}
	}
	\textbf{Output:} $X=(X^i)_{i\in V}$ that is distributed according to $\rho T_s^m$.
	\caption{Systematic scan MCMC sampler}
	\label{alg:generalMCMC-systematic}
\end{algorithm}

\begin{algorithm}[H]
	Sample $X\in\mathbb{S}$ from the distribution $\rho$;\\
	\For{$k=1,\ldots,nm$}{
	Sample $i\in V$ uniformly;\\
	Draw $Z^i$ from the distribution $K^{[i]}_x$;\\
	Set $X^i \leftarrow Z^i$;
	}
	\textbf{Output:} $X=(X^i)_{i\in V}$ that is distributed according to $\rho T_r^m$.
	\caption{Random scan MCMC sampler}
	\label{alg:generalMCMC-random}
\end{algorithm}

The following theorem describes the convergence behavior of the MCMC algorithms $T_s$ and $T_r$ under the so-called \emph{Dobrushin uniqueness condition}. Since the seminal work in \cite{Dob70} several authors have presented different approaches to establish convergence bounds of the type in Theorem \ref{thm:fastMCMC}. We refer to \cite{dyer2009} and references therein for a review of results that address fast mixing within the Dobrushin uniqueness framework. Theorem \ref{thm:fastMCMC} represents the building block for the theory that will be developed in the next sections.

\begin{theorem}[Local-update MCMC algorithms for discrete point processes]\label{thm:fastMCMC}
For each $i,j\in V$, define the \emph{Dobrushin coefficients} as:
\begin{align}
	C_{ij}
	:= \max_{x\in\mathbb{S}}
	| K^{[i]}_{x^{V\setminus\{j\}}0^j}(\{0\}) - K^{[i]}_{x^{V\setminus\{j\}}1^j}(\{0\}) |.
	\label{def:dobrushincoefficients}
\end{align}
Let $R\in\mathbb{R}^{n\times n}$ such that $C\le R$, element-wise. Assume that the \emph{Dobrushin uniqueness condition} holds, namely,
\begin{align}
	\| R \|_\infty
	:= \max_{i\in V} \sum_{j\in V} R_{ij} \le \gamma < 1.
	\label{def:Dobrushin condition}
\end{align}
Then, for any distribution $\rho$ on $\mathbb{S}$, any natural number $m\ge 0$, and any function $h:\mathbb{S}\rightarrow \mathbb{R}$, we have
\begin{align}
	| \rho T_s^m h - \mu h | \le 
	\gamma^m \sum_{i\in V} \delta_i (h),
	\qquad
	| \rho T_r^m h - \mu h | \le 
	\lambda^m \sum_{i\in V} \delta_i (h),
	\label{fast}
\end{align}
with $\lambda := e^{\gamma-1}<1$ and
$
	\delta_i (h) := \max_{x\in\mathbb{S}} 
	| h(x^{V\setminus\{i\}}0^i) - h(x^{V\setminus\{i\}}1^i) |.
$
\end{theorem}

\begin{remark}[On fast mixing]\label{rem:on fast mixing}
If $\gamma$ in the Dorbushin uniqueness condition \eqref{def:Dobrushin condition} does not depend on $n$, then the results in Theorem \ref{thm:fastMCMC} imply that the Markov chains defined by the transition kernels $T_s$ and $T_r$ are fast mixing. Recall that for each $\varepsilon > 0$ the mixing time $\tau(\varepsilon)$ of a Markov chain with transition kernel $T$ and unique invariant distribution $\mu$ is defined as \citep{levinmarkov}
$$
	\tau(\varepsilon) := 
	\min\left\{m\ge 0 : 
	\sup_{\rho \text{ distribution on } \mathbb{S}} 
	\|
	\rho T^m
	-
	\mu
	\|_{TV}
	\le \varepsilon
	 \right\},
$$
where $\| \rho - \rho' \|_{TV}:=\max_{A\subseteq\mathbb{S}} | \rho (A) - \rho' (A) |$ is the total variation distance between distributions $\rho$ and $\rho'$ on $\mathbb{S}$.
Clearly, from \eqref{fast} it follows that 
$
	|
	\rho T_s^m (A)
	-
	\mu (A)
	|
	\le 
	n \gamma^m
$
and
$
	|
	\rho T_r^m (A)
	-
	\mu (A)
	|
	\le 
	n \lambda^m,
$
for each $A\subseteq \mathbb{S}$.
Hence, we can easily derive the following upper bounds for the mixing time of the systematic scan and random scan Markov chains, respectively,
$$
	\tau_s(\varepsilon)
	\le \left\lceil\frac{\log(n \varepsilon^{-1})}{1- \gamma} \right\rceil,
	\qquad
	\tau_r(\varepsilon)
	\le \left\lceil\frac{\log(n \varepsilon^{-1})}{1- \lambda} \right\rceil,
$$
which show that the Markov chains are fast mixing, that is, their mixing time is upper bounded by a quantity that scales only logarithmically with the size $n$ of the set $V$.\footnote{Typically, fast mixing is defined when the mixing time is $O(n\log n)$, not $O(\log n)$. The difference in our setting is due to the definitions of the chains $T_s$ and $T_r$, which involve a full sweep over $n$ variables.}
In fact, taking the case of the systematic scan Markov chain, for instance, we have
\begin{align*}
	\tau(\varepsilon)
	\le \min\left\{m\ge 0 : 
	n \gamma^m 
	\le \varepsilon
	 \right\}
	= \min\left\{m\ge 0 : 
	m \ge \frac{\log(n \varepsilon^{-1})}{-\log \gamma}
	 \right\}
	 \le \left\lceil\frac{\log(n \varepsilon^{-1})}{1- \gamma} \right\rceil,
\end{align*}
where in the last inequality we used that $\log x \le x-1$ for each $x\in\mathbb{R}$.
\end{remark}

\begin{remark}[On probabilistic inference]\label{rem:On probabilistic inference}
If $\gamma$ in the Dorbushin uniqueness condition \eqref{def:Dobrushin condition} does not depend on $n$, then Theorem \ref{thm:fastMCMC} yields that the inference problem of approximating marginals of the type \eqref{goal:marginals} can be efficiently addressed via MCMC methods. In fact, choosing $h=1_A$, with $A=\{x\in\mathbb{S}:x^i=1, i\in S\}$ for a certain $S\subseteq V$, Theorem \ref{thm:fastMCMC} provides the following exponentially decreasing error bounds, for any $m\ge 0$ and any distribution $\rho$ on $\mathbb{S}$:
\begin{align*}
	| \rho T_s^m (A) - \mu (A) | \le 
	|S| \gamma^m,
	\qquad
	| \rho T_r^m (A) - \mu (A) | \le 
	|S| \lambda^m,
\end{align*}
which do not depend on $n$. As a consequence, if $X_1,\ldots,X_N$ are independent random variables distributed according to $\rho T_s^m$, generated as prescribed in Algorithm \ref{alg:generalMCMC-systematic} (analogous results follow for random scan), then $\hat \mu_A:=\frac{1}{N}\sum_{k=1}^N 1_A(X_k)$ is a biased estimator of $\mu (A)$ with the typical Monte Carlo mean square error bias/variance decomposition:
$$
	\mathbf{E} [ (\hat \mu_A - \mu (A))^2 ] =
	\underbrace{(\mathbf{E} [\hat \mu_A] - \mu (A) )^2}_{\text{bias}^2}
	+
	\underbrace{\mathbf{E} [(\hat \mu_A - \mathbf{E}[\hat \mu_A] )^2]}_{\text{variance}}
	\le (|S| \gamma^m)^2 + \frac{1}{N}.
$$
We stress once again that under the current assumptions this error bound does not depend on the set size $n$. This is in sharp contrast to the upper bounds for marginals produced by the theory developed for the variational methods in \cite{DK14}; these bounds, in fact, deteriorate exponentially with $n$, as it can be deduced by the very experimental results presented by the authors.
\end{remark}

\begin{remark}[On block updates]
Using the results in \cite{RvH14} it is possible to generalize Theorem \ref{thm:fastMCMC} for MCMC algorithms with blocks updates, in the spirit of \cite{Wei05}, that is, when the transition kernels $T^{[i]}$, $i\in V$, in \eqref{transition kernel} are replaced by
\begin{align*}
	T^{[S]}(x,z):=K^{[S]}(x,z^S) 1_{x^{V\setminus S}}(z^{V\setminus S}),
\end{align*}
for each \emph{block} $S\subseteq V$ in a given family $\mathcal{S}$ satisfying $\bigcup_{S\in\mathcal{S}} S = V$ (for instance, $\mathcal{S}=\{S\subseteq V: |S|=k\}$, for some $k\ge 1$), where $K^{[S]}$ is a probability kernel from $\mathbb{S}$ to $\bigtimes_{i\in S} \mathbb{S}_0= \{0,1\}^{|S|}$. In this case, for instance, a single step of the block-update random scan Gibbs sampler would read
$
	\frac{1}{|\mathcal{S}|} \sum_{S\in \mathcal{S}} T^{[S]},
$
which generalizes
$
	\frac{1}{n} \sum_{i \in V} T^{[i]}
$
in \eqref{MCMC-random}.
The new version of Theorem \ref{thm:fastMCMC} would yield weaker sufficient conditions for fast mixing involving higher values of the parameter $\beta$, in the spirit of \cite{DS85}. However, these conditions would be more convoluted and difficult to analyze, and they would involve a control on the ``block" derivatives of the form $\Delta_J\Delta_i f(S)$, with $J\subseteq V$. As the focus of the present work lies on investigating the \emph{basic} properties of the set function $f$ that makes it possible to efficiently address the inference problem in  \eqref{basic model} via MCMC methods, we limit our analysis to local sampling schemes with single site updates.
\end{remark}

\section{Condition for fast mixing and decay of correlations}\label{sec:Fast mixing local MCMC algorithms for discrete point processes}

In this section we introduce two of the most popular MCMC algorithms that are used in the literature within the local-update framework described by \eqref{transition kernel}: Gibbs sampler and Metropolis-Hastings \citep{asmussen2007stochastic}.
While many variants of these algorithms can also be considered, we restrict our analysis to their most basic implementations, as our goal is to investigate the fundamental principles behind fast mixing for discrete point processes.
For both of these algorithms we compute the Dobrushin coefficients \eqref{def:dobrushincoefficients}, and we provide a comparison between them in Lemma \ref{lem:comparison GS MH} below.
Then, we apply  Theorem \ref{thm:fastMCMC} to the analysis of the Gibbs sampler algorithm, and we present sufficient conditions for fast mixing, particularly in the case when $f$ is submodular, Lemma \ref{lem:Gibbs sampler fast mixing} below.
Henceforth, for each $x\in\mathbb{S}$ let $S(x)\subseteq V$ be defined as follows: $i\in S \text{ if } x(S)^i=1, i\not\in S \text{ if } x(S)^i=0$.

\begin{lemma}[Gibbs sampler]\label{lem:Gibbs sampler}
For each $i\in V$, in the transition kernel $T^{[i]}$ in \eqref{transition kernel} choose
\begin{align}
	K^{[i]}(x,z^i)
	&= \mathbf{E}[X^i= z^i | X^{V\setminus \{i\}} = x^{V\setminus \{i\}}]
	=
	\begin{cases}
	\frac{1}{1+e^{\beta \Delta_i f( S(x^{V\setminus \{i\}}0^i))}} &\text{if }z^i=0,\\
	\frac{e^{\beta \Delta_i f( S(x^{V\setminus \{i\}}0^i))}}
	{1+e^{\beta \Delta_i f( S(x^{V\setminus \{i\}}0^i))}} &\text{if }z^i=1,
	\end{cases}
	\label{def:Gibbssampler}
\end{align}
where the random variable $X$ has distribution $\mu$. Then, for each $i\in V$ we have $\mu T^{[i]} =\mu$, and the Dobrushin coefficients \eqref{def:dobrushincoefficients}  read, for each $i,j\in V$,
$$
	C_{ij}
	:=
	\begin{cases}
	0 &\text{if } i= j,\\
	\displaystyle
	\max_{S\subseteq V : S\not\ni i,j}
	\frac{|e^{\beta \Delta_i f(S)}-e^{\beta \Delta_i f(S\cup\{j\})}|}
	{(1+e^{\beta \Delta_i f(S)})(1+e^{\beta \Delta_i f(S\cup\{j\})})}
	&\text{if } i\neq j.
	\end{cases}
$$
\end{lemma}
	
\begin{lemma}[Metropolis-Hastings]\label{lem:Local Metropolis-Hastings}
For each $i\in V$, in the transition kernel $T^{[i]}$ in \eqref{transition kernel} choose
\begin{align}
	K^{[i]}(x,z^i) =
	\begin{cases}
	\alpha^i(x) &\text{if } z^i \neq x^i,\\
	1-\alpha^i(x)
	&\text{if } z^i =x^i,
	\end{cases}
	\qquad
	\alpha^i(x)
	:=
	\begin{cases}
	1 \land e^{-\beta\Delta_i f(S(x^{V\setminus \{i\}} 0^i))}
	&\text{if } x^i=1,\\
	1 \land e^{\beta\Delta_i f(S(x^{V\setminus \{i\}} 0^i))}
	&\text{if } x^i=0.
	\end{cases}
	\label{def:MH}
\end{align}
Then, for each $i\in V$ we have $\mu T^{[i]} =\mu$, and the Dobrushin coefficients \eqref{def:dobrushincoefficients}  read, for each $i,j\in V$,
\begin{align*}
	\widetilde C_{ij}\!:=\!
	\begin{cases}
	\displaystyle\max_{S\subseteq V : S\not\ni i} e^{- \beta |\Delta_i f(S)| }
	&\!\text{if } i= j,\\
	\displaystyle \max_{S\subseteq V : S\not\ni i,j}
	\left| 1 \!\land\! e^{\beta \Delta_i f(S)}
	\!-\! 1\land e^{\beta \Delta_i f(S\cup\{j\})} \right|
	\!\vee\!
	\left| 1 \!\land\! e^{-\beta \Delta_i f(S)}
	\!-\! 1 \!\land\! e^{-\beta \Delta_i f(S\cup\{j\})} \right| 
	&\!\text{if } i\neq j.
	\end{cases}
\end{align*}
\end{lemma}

The implementation of the Gibbs sampler \eqref{def:Gibbssampler} and Metropolis-Hastings \eqref{def:MH} is given, respectively, in Algorithm \ref{alg:gibbssampler} and Algorithm \ref{alg:Metropolis-Hastings} in Appendix \ref{sec:Proofs of MCMC convergence results} (we only present the random scan versions from Algorithm \ref{alg:generalMCMC-random}; the systematic scan versions can be obtained analogously from Algorithm \ref{alg:generalMCMC-systematic}).

The following Lemma compares the Dobrushin coefficients for the Gibbs sampler, Lemma \ref{lem:Gibbs sampler}, and the Metropolis-Hastings Markov chain, Lemma \ref{lem:Local Metropolis-Hastings}.

\begin{lemma}[Comparison of Dobrushin coefficients]\label{lem:comparison GS MH}
For each $i,j\in V, i\neq j$, we have
$$
	0=C_{ii} \le \widetilde C_{ii},
	\qquad \frac{1}{4} \widetilde C_{ij} \le C_{ij} \le \widetilde C_{ij}.
$$
\end{lemma}

From Lemma \ref{lem:comparison GS MH} it follows that $C\le \widetilde C$, element-wise. In this respect, it might seem a good idea to investigate conditions for fast mixing for the Metropolis Hastings algorithm, i.e., to upper bound $\widetilde C$, as these conditions will immediately yield fast mixing for the Gibbs sampler as well. However, this approach has a main drawback. In fact, while for each $i\in V$ it holds that $C_{ii}=0$, we have that $\widetilde C_{ii} \ge 0$, and for a large system (i.e., for a large value of $n=|V|$) we typically expect this quantity to be very close to $1$, which would create a problem to establish the Dobrushin uniqueness condition \eqref{def:Dobrushin condition}. For instance, take the case where the function $f$ is \emph{monotone} (i.e., $\Delta_{i}f(S) \ge 0$ for each $i\in V$, $S\subseteq V$) and submodular. Then,
$$
	\widetilde C_{ii}
	= \max_{S\subseteq V : S\not\ni i} e^{- \beta |\Delta_if(S)| }
	= \max_{S\subseteq V : S\not\ni i} e^{- \beta \Delta_if(S) }
	= e^{- \beta \Delta_if(V\setminus\{i\}) }.
$$
Typically, unless $f$ is modular or ``close" to modular (in the sense that its \emph{curvature} is close to zero, see Section \ref{sec:Modular functions}),
if $n$ is large we expect the marginal gain of $f$ when an element $i\in V$ is added to $V\setminus\{i\}$ to be close to $0$, that is, $\Delta_if(V\setminus\{i\}) \approx 0$, which yields $\widetilde C_{ii} \approx 1$ if $\beta \approx 1$. This issue is not present in the Gibbs sampler, as $C_{ii}=0$ for each $i\in V$. For this reason, henceforth we focus on the Gibbs sampler and we provide sufficient conditions for it to be fast mixing.

\begin{lemma}[Fast mixing Gibbs samplers for discrete point processes]\label{lem:Gibbs sampler fast mixing}
Assume that the following condition holds:
\begin{align}
	\alpha(\beta) \max_{i\in V} \sum_{j\in V\setminus\{i\}} 
	\max_{S\subseteq V : S\not\ni i,j}
	\left|
	1
	 -
	e^{\beta \Delta_j\Delta_if(S)}
	\right|
	\le \gamma < 1,
	\label{fastmixinggeneral}
\end{align}
where $\alpha(\beta) := \max_{i\in V} \max_{S\subseteq V\setminus\{i\}} e^{-\beta \Delta_if(S)}$. 
Let $T_s$ and $T_r$ be the Gibbs samplers defined as in Section \ref{sec:Fast mixing MCMC algorithms}, with $K^{[i]}$, $i\in V$, defined as in \eqref{def:Gibbssampler}. Then, for any distribution $\rho$ on $\mathbb{S}$, any natural number $m\ge 0$, and any function $h:\mathbb{S}\rightarrow \mathbb{R}$, we have
\begin{align*}
	| \rho T_s^m h - \mu h | \le 
	\gamma^m \sum_{i\in V} \delta_i (h),
	\qquad
	| \rho T_r^m h - \mu h | \le 
	\lambda^m \sum_{i\in V} \delta_i (h),
\end{align*}
where $\lambda := e^{\gamma-1}<1$ and
$
	\delta_i (h) := \max_{x\in\mathbb{S}} 
	| h(x^{V\setminus\{i\}}0^i) - h(x^{V\setminus\{i\}}1^i) |.
$
If $\gamma$ does not depend on $n$, then the chains are fast mixing.
In particular, if $f$ is submodular\footnote{As done in \cite{DK14} and \cite{iyer2015-spps}, we can also specialize our results to \emph{supermodular} functions ($f$ is supermodular if $-f$ is submodular). In this case we would get $\alpha(\beta) = \max_{i\in V} e^{-\beta \Delta_if(\varnothing)}$.}, then condition \eqref{fastmixinggeneral} reads
\begin{align}
	\alpha(\beta) \max_{i\in V} \sum_{j\in V\setminus\{i\}} 
	\max_{S\subseteq V : S\not\ni i,j}
	\left(
	1
	 -
	e^{\beta \Delta_j\Delta_if(S)}
	\right) 
	\le \gamma < 1,
	\label{fastmixingsubmodular}
\end{align}
with $\alpha(\beta) = \max_{i\in V} e^{-\beta \Delta_if(V\setminus\{i\})}$. If $f$ is monotone, clearly $\alpha(\beta)\le 1$.
\end{lemma}

Note that using the inequality $1-e^{x} \le -x$ for each $x\in\mathbb{R}$, condition \eqref{fastmixinggeneral} can be replaced by the stronger (but simpler) condition 
$
	\alpha(\beta) \beta
	\|M\|_\infty
	 \le \gamma < 1,
$
where $M_{ij} := \max_{S\subseteq V: S\not\ni i,j} |\Delta_j\Delta_i f(S)|$ for each $i\neq j$, and $M_{ii}:=0$ for each $i$.
This condition makes it manifest the property of the function $f$ that renders the inference problem tractable via local-update MCMC methods. As highlighted in Remark \ref{rem:on fast mixing} and also in Remark \ref{rem:On probabilistic inference}, in order for the Markov chains to be fast mixing we need $\gamma$ to be independent of the set size $n=|V|$. In particular, we need $\|M\|_\infty$ to be upper bounded by a constant that does not depend on $n$. This is achieved, for instance, if $(V,d)$ is a metric space and the function $f$ displays one of the following forms of \emph{decay of correlation} with respect to the metric $d$.
\begin{itemize}
\item (Exponential decay of correlations) Assume that for each $i,j\in V, i\neq j$ we have
$$
	M_{ij} := \max_{S\subseteq V : S\not\ni i,j} |\Delta_j\Delta_if(S)| 
	\le \alpha e^{-\alpha'd(i,j)},
$$
where $\alpha,\alpha'>0$ are two constants so that $\|M\|_\infty \le \max_{i\in V} \sum_{j\in V} \alpha e^{-\alpha'd(i,j)} \le c$, where $c$ does not depend on $n$.
\item (Finite-range correlations) Assume that for each $i,j\in V, i\neq j$, there exists $r> 0$ such that 
$$
	M_{ij} := \max_{S\subseteq V : S\not\ni i,j} |\Delta_j\Delta_if(S)|
	\le
	\begin{cases}
	c &\text{if } d(i,j) \le r,\\
	0 &\text{if } d(i,j) > r,
	\end{cases}
$$
where $c\ge 0$ does not depend on $n$. For each $i\in V$ define
$
	N(i) := \{j\in V: d(i,j)\le r\},
$
and assume that $N:=\max_{i\in V} |N(i)|$ does not depend on $n$. Then, clearly, $\|M\|_\infty \le cN$.
\end{itemize}

In the remaining of this paper we assume that the function $f$ is submodular, and we investigate condition \eqref{fastmixingsubmodular}. In this case, note that the quantity $\alpha(\beta)$ involves an optimization over only $n$ values, while for each $i,j\in V$, $i\neq j$, the term
$
	\max_{S\subseteq V : S\not\ni i,j} \left(
	1
	 -
	e^{\beta \Delta_j\Delta_if(S)}
	\right)
$
involves an optimization problem over $2^{n-2}$ possibilities. As we will see in the next section, however, it is often the case that we can compute this term exactly, or that we can easily produce upper bounds for it.

\section{Applications to log submodular point processes}\label{sec:applications}

In this section we apply Lemma \ref{lem:Gibbs sampler fast mixing} to a few canonical examples of submodular functions. First, in Section \ref{sec:Modular functions} we discuss the elementary behavior of the Gibbs sampler in trivial (i.e., i.i.d.) models defined by modular functions. This case will serve to point out another deficiency (on top of the one already highlighted in Remark \ref{rem:On probabilistic inference}) of the theoretical results presented in \cite{DK14} with respect to the notion of curvature.
Consecutively, we consider functions that are defined on a finite graph $G=(V,E)$, where $n=|V|$ and the edge set $E$ describes the pair of vertices that interact, Sections \ref{sec:Facility location}, \ref{sec:Generalized graph cut}, and \ref{sec:Determinantal point processes}. The strength of the pairwise interaction is modeled by a symmetric matrix $L$. The nature of the assumptions that our theory requires to yield fast mixing depends on the application at hand, and they involve the structure of the matrix $L$. Note that for the applications discussed in Sections \ref{sec:Facility location} and \ref{sec:Generalized graph cut} these assumptions would be satisfies if $L_{ij} \lesssim e^{-d(i,j)}$, or if $L_{ij} = 0$ whenever $d(i,j)>r$, in the spirit of the decay of correlation properties discussed in the previous section. 
Finally, in Section \ref{sec:Decomposable functions} we consider the case of decomposable functions, that is, functions that can be represented as sums of concave functions applied to modular functions. Since we need to calculate discrete derivatives, we usually prove submodularity along the way as it is nothing but $\Delta_j\Delta_i f(S)\leq 0$ for each $i,j\in V$, $i\neq j$, $S\subseteq V$, $i,j\not\in S$.

\subsection{Modular functions}\label{sec:Modular functions}
Let $V$ be a finite set with $n=|V|$. For each $i\in V$, let $w_i\in\mathbb{R}$ be given. For each $S\subseteq V$, let
$
	f(S) := \sum_{i\in S} w_i
$
and
$
	f(\varnothing):=0.
$
It is immediate to verify that the function $f$ is modular as $\Delta_j\Delta_if(S)=0$ for each $i,j\in V$, $i\neq j$, $S\subseteq V$. Clearly, in this case condition \eqref{fastmixingsubmodular} is satisfied with $\gamma=0$, and Lemma \ref{lem:Gibbs sampler fast mixing} implies that for each $m \ge 1$ we have
$
	| \rho T_s^m h - \mu h | = 0,
$
and
$
	| \rho T_r^m h - \mu h | = 0.
$
In particular, these results hold even for $m=1$, which means that the Gibbs samplers sample \emph{exactly} from the distribution $\mu$ in a single sweep.

We now slightly tweak the trivial example just considered to compare the theoretical guarantees provided in Lemma \ref{lem:Gibbs sampler fast mixing} for the Gibbs sampler against the guarantees provided in \cite{DK14} for the variational methods they proposed, as a function of the \emph{curvature} $c(f)$ of $f$, which, for monotone submodular functions with $f(\varnothing)\ge 0$, is defined as:
$$
	c(f) := 1 - \min_{i\in V:f(\{i\})\neq 0}\frac{\Delta_{i}f(V\setminus\{i\})}{f(\{i\})}.
$$
Fix $k,k'\in V$, $k\neq k'$, and let $g$ be the function defined as $g(S):=1$ if $S\cap\{k,k'\}\neq \varnothing$ and $g(S):=0$ if $S\cap\{k,k'\}= \varnothing$. It is easy to check that $g$ is submodular with
$$
	\Delta_j\Delta_i g(S) =
	\begin{cases}
	-1 &\text{if } \{i,j\}=\{k,k'\},\\
	0 &\text{otherwise},
	\end{cases}
$$
for each $i,j\in V$, $i\neq j$, $S\subseteq V$, $i,j\not\in S$.
For each $i\in V\setminus\{k,k'\}$ let $w_i=1$, and let $w_k=w_{k'}=0$. Consider the submodular function
$
	f(S) := \sum_{i\in S} w_i + g(S) = |S \setminus \{k,k'\}| + g(S).
$
Then we clearly have $c(f)=1$, as the minimum that appears in the definition of curvature is $0$ (it is attained at $i=k$, for instance), so that the bounds in \cite{DK14} diverge to infinity. On the other hand, in this case the results in Lemma \ref{lem:Gibbs sampler fast mixing} hold with $\gamma=1 - e^{-\beta}$. In fact, $\alpha(\beta) = \max_{i\in V} e^{-\beta \Delta_if(V\setminus\{i\})} =1$ and $\Delta_j\Delta_i f(S)= \Delta_j\Delta_i g(S)$, from which it follows that condition \eqref{fastmixingsubmodular} reads
$$
	\alpha(\beta)
	\max_{i\in V} \sum_{j\in V\setminus\{i\}} 
	\max_{S\subseteq V : S\not\ni i,j}
	\left(
	1
	 -
	e^{\beta \Delta_j\Delta_if(S)}
	\right) 
	= 
	1 - e^{-\beta}
	=: \gamma < 1.
$$

\subsection{Facility location}\label{sec:Facility location}
Let $V=\{1,\ldots,n\}$ be a collection of facilities, and $W=\{1,\ldots,m\}$ be a collection of customers. Let $L_{k\ell}\ge 0$ be the value provided to costumer $k\in V$ by the facility $\ell \in W$. For each $S\subseteq V$, let
$
	g(S) := \sum_{k=1}^m \max_{\ell \in S} L_{k\ell}
$
and
$
	g(\varnothing):=0.
$
By definition, $g(S)$ is the maximum value that the facilities in $S$ can provide to all customers, provided that each customer chooses to be served only the facility that provides the highest value to him. We consider the function $f$ defined as
\begin{align}
	f(S) := g(S) - \lambda |S|,
	\label{f:facility}
\end{align}
where $|S|$ denotes the cardinality of the set $S$, and $\lambda\ge 0$. Such a function is submodular and has been considered in many applications such as large-scale clustering \citep{mirzasoleiman2013distributed} and recommender systems \citep{el2009turning}. 

\begin{corollary}\label{cor:facility}
For the function $f$ in \eqref{f:facility}, Lemma \ref{lem:Gibbs sampler fast mixing} applies if condition \eqref{fastmixingsubmodular} is replaced with
\begin{align*}
	e^{\lambda\beta} \max_{i\in V} \sum_{j\in V\setminus\{i\}} 
	\left(
	1
	 -
	e^{-\beta \sum_{k=1}^m L_{ki} \wedge L_{kj}}
	\right) 
	\le \gamma < 1.
\end{align*}
\end{corollary}

\subsection{Generalized graph cut}\label{sec:Generalized graph cut}
Let $V$ be a vertex set, and for each $i,j \in V, i\neq j$, let $L_{ij}=L_{ji}\ge 0$ be the weight associated to the undirected edge between $i$ and $j$. Let $L_{ii}=0$ for each $i \in V$. For each $S\subseteq V$, let
\begin{align}
	f(S) := 
	a + b\sum_{k\in S} \sum_{\ell \in V} L_{k\ell}
	- c \sum_{k\in S} \sum_{\ell \in S} L_{k\ell}
	\label{f:cut}
\end{align}
and
$
	f(\varnothing) =a,
$
with $a,b,c\ge 0$ (typically $a$ is chosen so that $f(S)\ge 0$ for any $S\subseteq V$). In the case $a=0$, $b=c=1$, we recover
$
	f(S) = f(V\setminus S) = 
	\sum_{k\in S} \sum_{\ell \in V\setminus S} L_{k\ell}
$
and
$
	f(\varnothing) = f(V)=0,
$
which is the standard graph cut function. Namely, $f(S)$
is the sum of the weights of each edge that connects a point in $S$ with a point in $V\setminus S$. The generalized cut function $f$ defined above is submodular, with many applications in computer vision \citep{jegelka2011fast}.

\begin{corollary}\label{cor:graphcut}
For the function $f$ in \eqref{f:cut}, Lemma \ref{lem:Gibbs sampler fast mixing} applies if condition \eqref{fastmixingsubmodular} is replaced with
\begin{align*}
	e^{\beta (2c-b) \min_{i \in V} \sum_{\ell\in V\setminus\{i\}} L_{i\ell}}
	\max_{i\in V} \sum_{j\in V\setminus\{i\}} 
	\left(
	1 - e^{-2 c \beta L_{ij}}
	\right)
	\le \gamma < 1.
\end{align*}
\end{corollary}

\subsection{Determinantal point processes}\label{sec:Determinantal point processes}
Fix $n$ and let $L\in\mathbb{R}^{n\times n}$ be a positive definite matrix. Let $V=\{1,\ldots,n\}$, and for each $S\subseteq V$ let
\begin{align}
	f(S) := \log\det L_S
	\label{f:logdet}
\end{align}
and
$
	f(\varnothing):=0,
$
where $L_S := (L_{ij})_{i,j\in S}$. Such a function is submodular and has been used in determinantal point processes for which the partition function can be computed exactly in time $O(n^3)$.  Recently, several authors have considered MCMC algorithms to sample from determinant point processes, see \cite{SG13} and \cite{K13} for instance. However, to the best of our knowledge, no theoretical guarantees have ever been established for the MCMC\footnote{There are efficient sampling schemes that are not based on the MCMC paradigm. See \cite{DR10} and discussion therein, where they explicitly pose the problem of designing MCMC samplers.}  algorithms being adopted. As mentioned in the introduction, the proofs in \cite{K13} are wrong (although the sampling scheme is correct, exactly matching Algorithm \ref{alg:Metropolis-Hastings} in the case $\beta=1$).

In this section we give a probabilistic interpretation of the mechanism behind the structure of the matrix $L$ that guarantees fast mixing MCMC algorithms in the light of the theory being developed.
Let $L$ be the covariance matrix of a collection of Gaussian random variables $X^1,\ldots,X^n$. It can be shown that, for $i,j\in V$, $i\neq j$, $S\subseteq V$, $i,j\not\in S$, we have
$
	\operatorname{Cov}(X^i,X^j|X^S) := L_{ij} - L_{i,S} L^{-1}_{S} L_{S,j}.
$
Define the conditional correlation coefficients as follows, for $i,j\in V$, $S\subseteq V$:
$$
	\rho(i,j|S) = \frac{\operatorname{Cov}(X^i,X^j|X^S)}
	{\sqrt{\operatorname{Var}(X^i|X^S)} \sqrt{\operatorname{Var}(X^j|X^S)}},
$$
and let $I(X^i;X^j|X^S)$ be the conditional mutual information of $X^i$ and $X^j$ given $X^S$. It holds:
\begin{align}
	I(X^i;X^j|X^S) = -\frac{1}{2} \log (1-\rho(i,j|S)^2).
	\label{def:condmutualinfo}
\end{align}
We now show how condition \eqref{fastmixingsubmodular} in Lemma \ref{lem:Gibbs sampler fast mixing} can be stated in terms of the quantities just introduced.

\begin{corollary}\label{cor:ddp}
For the function $f$ in \eqref{f:logdet}, Lemma \ref{lem:Gibbs sampler fast mixing} applies with the condition \eqref{fastmixingsubmodular} that reads
\begin{align*}
	\max_{i\in V} \frac{1}{\operatorname{Var}(X^i|X^{V\setminus\{i\}})^{\beta}}
	\max_{i\in V} \sum_{j\in V\setminus\{i\}} 
	\max_{S\subseteq V : S\not\ni i,j}
	\left(1 - e^{-2 \beta I(X^i;X^j|X^S)}\right)
	\le \gamma < 1.
\end{align*}
If $\beta=1$, the condition is
$
	\max_{i\in V}\! \frac{1}{\operatorname{Var}(X^i|X^{V\setminus\{i\}})}
	\!\max_{i\in V} \!\sum_{j\in V\setminus\{i\}} 
	\!\max_{S\subseteq V : S\not\ni i,j}
	\!\rho(i,j|S)^2
	\!\le\! \gamma \!<\! 1.
$
\end{corollary}

Practitioners who are interested in fast mixing Gibbs samplers for determinant point processes need to exploit additional structure in the matrix $L$ to simplify the conditions in Corollary \ref{cor:ddp}.

\subsection{Decomposable functions}\label{sec:Decomposable functions}
Given a finite set $V$, let $\mathcal{S}$ be a collection of subsets of $V$ that covers $V$, i.e., $\mathcal{S}\subseteq 2^V$ with $\bigcup_{A\in\mathcal{S}}A = V$. For each $A\in\mathcal{S}$ let $\phi_A: \mathbb{R}_+\rightarrow\mathbb{R}$ be a concave function. For each $S\subseteq V$, let
\begin{align}
	f(S) := \sum_{A\in\mathcal{S}} \phi_A(|A\cap S|).
	\label{f:decomposable}
\end{align}
The function $f$ is submodular \citep{SK10}. For instance, note that if $\mathcal{S}=\{S\subseteq V : |S|=2\}$ and for each $A\in\mathcal{S}$ we have $\phi_A(0)=\phi_A(2)=-J/n$, $\phi_A(1)=J/n$, for a given constant $J>0$, then \eqref{basic model} corresponds to the distribution of an antiferromagnetic mean-field Ising model with inverse temperature $\beta$ and zero external magnetic field.

\begin{corollary}\label{cor:decomposable}
For each $A\in \mathcal{S}$ let the concave function $\phi_A$ be twice differentiable. Assume there exist constants $c'<0<c$ such that $ d\phi_A(x)/dx \geq c$ and $c' \leq d^2\phi_A(x)/dx^2\leq 0$ for all $x\in [0,n]$ and $A\in\mathcal{S}$. 
Then, for the function $f$ in \eqref{f:decomposable}, Lemma \ref{lem:Gibbs sampler fast mixing} applies if condition \eqref{fastmixingsubmodular} is replaced with
\begin{align*}
	(1-e^{c'\beta})
	e^{-c \beta \min_{i\in V}  |\{A\in \mathcal{S}: A\ni i\}|} 
	\max_{i\in V} \left|\bigcup_{A\in\mathcal{S}:A\ni i} A\right|
	\le \gamma < 1.
\end{align*}
\end{corollary}

\acks{We would like to thank Sekhar Tatikonda for useful discussions.}

\bibliography{bib}

\appendix

\section{Theory, proofs}\label{sec:Proofs of MCMC convergence results}

Below are the proofs of the results presented in Section \ref{sec:Fast mixing MCMC algorithms} and Section \ref{sec:Fast mixing local MCMC algorithms for discrete point processes}.\\

\begin{proof}[Proof of Theorem \ref{thm:fastMCMC}]
The proof is based on the Wasserstein matrix approach, which is a standard tool in the analysis of high-dimensional Markov chains, cf. \cite{Fol79}.
For each two distributions $\rho,\rho'$ on $\mathbb{S}_0$, and function $g:\mathbb{S}_0\rightarrow\mathbb{R}$, we have
$$
	|\rho g - \rho' g|
	\le | g(0) - g(1) | | \rho(\{0\})-\rho'(\{0\}) |,
$$
from which it follows that, for each $i,j\in V$,
\begin{align*}
	\delta_{j} (K^{[i]} g)
	&=\max_{x\in\mathbb{S}} 
	| K^{[i]}g(x^{V\setminus\{j\}}0^j) - K^{[i]}g(x^{V\setminus\{j\}}1^j) |
	\le | g(0) - g(1) | C_{ij}.
\end{align*}
For each function $h:\mathbb{S}\rightarrow\mathbb{R}$, $i\in V$, $x\in\mathbb{S}$, define the function $h^i_x:z^i\in\mathbb{R}\longrightarrow h^i_x(z^i):=h(z^ix^{V\setminus\{i\}})$. Then, we have
\begin{align*}
	&\delta_j(T^{[i]}h)
	=\max_{x\in\mathbb{S}} 
	| K^{[i]}_{x^{V\setminus\{j\}}0^j}h^i_{x^{V\setminus\{j\}}0^j} 
	- K^{[i]}_{x^{V\setminus\{j\}}1^j}h^i_{x^{V\setminus\{j\}}1^j} |\\
	&\quad\le \max_{x\in\mathbb{S}} 
	| K^{[i]}_{x^{V\setminus\{j\}}0^j}h^i_{x^{V\setminus\{j\}}0^j}
	- K^{[i]}_{x^{V\setminus\{j\}}1^j}h^i_{x^{V\setminus\{j\}}0^j} | 
	+\max_{x\in\mathbb{S}} 
	| K^{[i]}_{x^{V\setminus\{j\}}1^j}(h^i_{x^{V\setminus\{j\}}0^j}
	- h^i_{x^{V\setminus\{j\}}1^j}) |\\
	&\quad\le \max_{y\in\mathbb{S}}
	\delta_j(K^{[i]}h^i_{y}) 
	+\max_{x\in\mathbb{S}}\max_{z\in\mathbb{S}_0} 
	| h^i_{x^{V\setminus\{j\}}0^j}(z) - h^i_{x^{V\setminus\{j\}}1^j}(z) |\\
	&\quad\le \delta_i(h) C_{ij} + \delta_j(h) 1_{i\neq j},
\end{align*}
where in the last line we applied the previous bound, and we use the notation $1_{i\neq j} := 1$ if $i\neq j$, $1_{i\neq j} := 0$ otherwise.
It is convenient to rewrite the previous estimate using a matrix notation:
\begin{align}
	\delta_j(T^{[i]}h) \le \sum_{k\in V} \delta_k(h) (W^{[i]})_{kj},
	\label{lanford estimate}
\end{align}
where for each $i\in V$, the matrix $W^{[i]}\in\mathbb{R}^{n\times n}$ is defined as
$
	W^{[i]}_{kj} := 1_{k=i} C_{kj} + 1_{k\neq i}1_{k=j}.
$
The matrix $W^{[i]}$ is a \emph{Wasserstain matrix} for the transition kernel $T^{[i]}$.

First, we apply estimate \eqref{lanford estimate} to bound the systematic scan Markov chain. Iterating this estimate, we immediately find
$
	\delta_j(T_sh) \le \sum_{k\in V} \delta_k(h) (W_s)_{kj},
$
where $W_s:=W^{[i_n]}\cdots W^{[i_1]}$.
Given any two distributions $\rho$ and $\rho'$ on $\mathbb{S}$, by a telescoping argument we find, for any $m\ge 0$,
\begin{align*}
	| \rho T_s^m h - \rho' T_s^m h |
	&\le \max_{x,z\in\mathbb{S}} | T_s^m h(x) - T_s^m h(z) |
	\le \sum_{j\in V} \delta_j (T_s^mh)
	\le \sum_{i\in V} \delta_i(h) \sum_{j\in V} (W_s^m)_{ij}\\
	&\le \| W_s^m \|_\infty \sum_{i\in V} \delta_i(h)
	\le \| W_s \|_\infty^m \sum_{i\in V} \delta_i(h)
	\le \gamma^m \sum_{i\in V} \delta_i(h),
\end{align*}
where for the last inequality we used that $\| W_s \|_\infty \le \| C \|_\infty \le \gamma < 1$ (this fact is proved in Corollary 24 in \cite{dyer2009}; note that the $\| \cdot \|_1$ matrix norm in the authors's notation corresponds to the $\| \cdot \|_\infty$ matrix norm in our notation).

We now apply estimate \eqref{lanford estimate} to bound the random scan Markov chain. 
As an intermediate step, define $S:=\frac{1}{n}\sum_{i\in V}{T^{[i]}}$. We have
$$
	\delta_j(Sh) 
	\le \frac{1}{n} \sum_{i\in V} \delta_j(T^{[i]}h)
	\le \frac{1}{n} \sum_{i\in V} \sum_{k\in V} \delta_k(h) (W^{[i]})_{kj}
	= \sum_{i\in V} \delta_i(h) Z_{ij},
$$
where $Z:=\frac{1}{n} C + \frac{n-1}{n} I \in\mathbb{R}^{n\times n}$, $I$ being the identity matrix. Iterating this result $n$ times, we find
$
	\delta_j(T_rh) 
	\le \sum_{i\in V} \delta_i(h) (W_r)_{ij},
$
where $W_r = Z^n = (\frac{1}{n} C + \frac{n-1}{n} I)^n$.
By the same argument presented above, we get
\begin{align*}
	| \rho T_r^m h - \rho' T_r^m h |
	\le \| W_r \|_\infty^m \sum_{i\in V} \delta_i(h)
	\le \left(1-\frac{1-\gamma}{n}\right)^{nm} \sum_{i\in V} \delta_i(h)
	\le (e^{\gamma-1})^m \sum_{i\in V} \delta_i(h).
\end{align*}
The proof is concluded if we take $\rho'$ to be $\mu$, noticing that $\mu T_s = \mu T_r = \mu$.
\end{proof}

\begin{proof}[Proof of Lemma \ref{lem:Gibbs sampler}]
Each kernel $T^{[i]}$ leaves $\mu$ invariant as a consequence of the tower property of conditional expectations as, for each function $h:\mathbb{S}\rightarrow\mathbb{R}$ we have
$$
	 \mu T^{[i]} h
	 = \mathbf{E} [T^{[i]} h(X)]
	 = \mathbf{E}[ \mathbf{E}[h(X)|X^{V\setminus\{v\}}] ] 
	 = \mathbf{E}[ h(X) ] 
	 = \mu h.
$$
For $i,j\in V$, $i\neq j$, $x\in\mathbb{S}$, we have
\begin{align*}
	| K^{[i]}_{x^{V\setminus \{j\}}0^j}(\{0\}) - K^{[i]}_{x^{V\setminus \{j\}}1^j}(\{0\}) |
	&=
	\left|
	\frac{1}{1+e^{\beta\Delta_if(S(x^{V\setminus \{i,j\}}0^i0^j))}}
	-
	\frac{1}{1+e^{\beta\Delta_if(S(x^{V\setminus \{i,j\}}0^i1^j))}}
	\right|\\
	&= 
	\frac{|e^{\beta\Delta_i f (S(x^{V\setminus \{i,j\}}0^i0^j))}
	-e^{\beta\Delta_if(S(x^{V\setminus \{i,j\}}0^i1^j))}|}
	{(1+e^{\beta\Delta_i f (S(x^{V\setminus \{i,j\}}0^i0^j))})
	(1+e^{\beta\Delta_i f(S(x^{V\setminus \{i,j\}}0^i1^j))})}.
\end{align*}
The proof is concluded taking the maximum over $x\in\mathbb{S}$ on both hand sides.
\end{proof}

\begin{proof}[Proof of Lemma \ref{lem:Local Metropolis-Hastings}]
The fact that each kernel $T^{[i]}$ leaves $\mu$ invariant follows as a consequence of the so-called \emph{detailed balance equation}, which holds for each $x\in\mathbb{S},y^i\in\mathbb{S}_0$ and is immediately verified (once noticed that $\alpha^i(x) = 1 \land \frac{\mu(x^{V\setminus \{i\}}(1-x^i)^i)}{\mu(x)}$):
$$
	\mu(x) K^{[i]}(x,y^i) = \mu(x^{V\setminus\{i\}}y^i) K^{[i]}(x^{V\setminus\{i\}}y^i,x^i).
$$
In fact, this equation yields, for each function $h:\mathbb{S}\rightarrow\mathbb{R}$,
\begin{align*}
	\mu T^{[i]} h
	&= \sum_{x\in\mathbb{S}}
	\sum_{y^i\in\mathbb{S}_0}
	\mu(x) K^{[i]}(x,y^i)
	h(x^{V\setminus\{i\}}y^i)\\
	&= \sum_{x\in\mathbb{S}}
	\sum_{y^i\in\mathbb{S}_0}
	\mu(x^{V\setminus\{i\}}y^i) K^{[i]}(x^{V\setminus\{i\}}y^i,x^i)
	h(x^{V\setminus\{i\}}y^i),
\end{align*}
and the right-hand side of this expression is equal to $\mu h$, as clearly $\sum_{x^i\in\mathbb{S}_0} K^{[i]}(x^{V\setminus\{i\}}y^i,x^i)=1$.\\
For $i\in V$, $x\in\mathbb{S}$, we have
\begin{align*}
	| K^{[i]}_{x^{V\setminus\{i\}}0^i}(\{0\}) - K^{[i]}_{x^{V\setminus\{i\}}1^i}(\{0\}) |
	&= | 1-\alpha^i(x^{V\setminus\{i\}}0^i) - \alpha^i(x^{V\setminus\{i\}}1^i) |
	= e^{-\beta |\Delta_i f(S(x^{V\setminus \{i\}} 0^i))|}.
\end{align*}
For $i,j\in V$, $i\neq j$, $x\in\mathbb{S}$, we have
\begin{align*}
	| K^{[i]}_{x^{V\setminus \{i,j\}}0^i0^j}(\{0\}) \!-\! K^{[i]}_{x^{V\setminus \{i,j\}}0^i1^j}(\{0\}) |
	&= | 1 \!\land\! e^{\beta \Delta_if(S(x^{V\setminus \{i,j\}}0^i0^j))}
	\!-\! 1\!\land\! e^{\beta \Delta_i f(S(x^{V\setminus \{i,j\}}0^i1^j))} |,\\
	| K^{[i]}_{x^{V\setminus \{i,j\}}1^i0^j}(\{0\}) \!-\! K^{[i]}_{x^{V\setminus \{i,j\}}1^i1^j}(\{0\}) |
	&= | 1 \!\land\! e^{-\beta \Delta f_i(S(x^{I\setminus \{i,j\}}0^i0^j))}
	\!-\! 1 \!\land\! e^{-\beta \Delta_i f(S(x^{I\setminus \{i,j\}}0^i1^j))} |.
\end{align*}
Hence, the Dobrushin coefficients \eqref{def:dobrushincoefficients} read
\begin{align*}
	&\widetilde C_{ij}= \max_{x\in\mathbb{S}}
	| K^{[i]}_{x^{V\setminus\{j\}}0^j}(\{0\}) - K^{[i]}_{x^{V\setminus\{j\}}1^j}(\{0\}) | =\\
	&
	\begin{cases}
	\displaystyle\max_{S\subseteq V : S\not\ni i} e^{- \beta |\Delta_i f(S)| }
	&\text{if } i= j,\\
	\displaystyle \max_{S\subseteq V : S\not\ni i,j}
	\left| 1 \land e^{\beta \Delta_i f(S)}
	- 1\land e^{\beta \Delta_i f(S\cup\{j\})} \right|
	\vee
	\left| 1 \land e^{-\beta \Delta_i f(S)}
	- 1 \land e^{-\beta \Delta_i f(S\cup\{j\})} \right| 
	&\text{if } i\neq j.
	\end{cases}
\end{align*}
\end{proof}

Below are the algorithms considered in Section \ref{sec:Fast mixing local MCMC algorithms for discrete point processes}.
\begin{algorithm}[H]
	Sample $S \subseteq V$ from the distribution $\rho$;\\
	\For{$k=1,\ldots,nm$}{
	Sample $i\in V$ uniformly;\\
	Draw $C\in\{0,1\}$ with
	$\mathbf{P}(C=0) = \frac{1}
	{1+e^{\Delta_i f(S\setminus\{i\})}}$;\\
	If $C=0$ then set $S \leftarrow S\setminus\{i\}$,
	else set $S \leftarrow S\cup\{i\}$;
	}
	\textbf{Output:} $S\subseteq V$ that is distributed according to $\rho T_r^m$.
	\caption{Random scan Gibbs sampler}
	\label{alg:gibbssampler}
\end{algorithm}

\begin{algorithm}[H]
	Sample $S \subseteq V$ from the distribution $\rho$;\\
	\For{$k=1,\ldots,nm$}{
	Sample $i\in V$ uniformly;\\
	\eIf{$i\in S$}{
		draw $C\in\{0,1\}$ with
		$\mathbf{P}(C=0) = 1 \land e^{-\beta\Delta_i f(S\setminus\{i\})}$;
		if $C=0$, then set $S \leftarrow S\setminus\{i\}$;\;
		}{
		draw $C\in\{0,1\}$ with
		$\mathbf{P}(C=1) = 1 \land e^{\beta\Delta_i f(S\setminus\{i\})}$;
		if $C=1$, then set $S \leftarrow S\cup\{i\}$;\;
		}
  	}
	\textbf{Output:} $S\subseteq V$ that is distributed according to $\rho T_r^m$.
	\caption{Random scan Metropolis-Hastings}
	\label{alg:Metropolis-Hastings}
\end{algorithm}

\begin{proof}[Proof of Lemma \ref{lem:comparison GS MH}]
The first statement is trivial as $C_{ii} =0$ and $\widetilde C_{ii}\ge 0$ by definition. To prove the second statement, we show that the following holds for each pair of real numbers $a,b$ with $b \le a$:
$$
	\frac{1}{4} h(a,b) \le
	\frac{e^a-e^b}{(1+e^a)(1+e^b)}
	\le h(a,b),
$$
where
$
	h(a,b) :=
	(1\land e^a - 1\land e^b) 
	\vee
	(1\land e^{-b} - 1\land e^{-a}).
$
We distinguish three cases.

If $0\le b\le a$, then $h(a,b)=e^{-b}-e^{-a}$, and we have
$$
	\frac{1}{4} h(a,b)
	= \frac{e^{-b}-e^{-a}}{4}
	= \frac{e^a-e^b}{4 e^a e^b}
	\le
	\frac{e^a-e^b}{(1+e^a)(1+e^b)}
	\le \frac{e^a-e^b}{e^a e^b}
	= e^{-b}-e^{-a} = h(a,b).
$$

If $b \le 0\le a$, then $h(a,b) = (1-e^b)\vee (1-e^{-a})= 1-e^{(-a) \land b}$. As
$$
	\frac{1-e^{2((-a) \land b)}}{2}
	\le \frac{1-e^{2((-a) \land b)}}
	{1+e^{(-a) \land b - (-a) \vee b}}
	\le 1-e^{(-a) \land b + (-a) \vee b}
	= 1-e^{-a + b},
$$
where for the second inequality we used that
$
	\frac{1-e^{2x}}{1+e^{x-y}} \le 1-e^{x+y}
$
for $x\le y\le 0$,
it follows that
$$
	\frac{1}{4} h(a,b)
	\le \frac{1}{2} \frac{1-e^{(-a) \land b}}{1+e^{(-a) \vee b}}=
	\frac{1}{2} \frac{1-e^{2((-a) \land b)}}{(1+e^{-a})(1+e^b)}
	\le \frac{1-e^{-a + b}}{(1+e^{-a})(1+e^b)} 
	= \frac{e^a-e^b}{(1+e^a)(1+e^b)}.
$$
Moreover,
$$
	\frac{e^a-e^b}{(1+e^a)(1+e^b)}
	= \frac{1-e^{-a + b}}{(1+e^{-a})(1+e^b)}
	\le \frac{1-e^{2((-a) \land b)}}{(1+e^{(-a) \land b})(1+e^{(-a) \vee b})}
	= \frac{1-e^{(-a) \land b}}{1+e^{(-a) \vee b}}
	\le h(a,b).
$$

Finally, if $b\le a\le 0$, then $h(a,b)=e^{a}-e^{b}$, and we have
$$
	\frac{1}{4} h(a,b)
	= \frac{e^a-e^b}{4}
	\le \frac{e^a-e^b}{(1+e^a)(1+e^b)}
	\le e^a-e^b = h(a,b).
$$
\end{proof}

\section{Applications, proofs}\label{sec:Proofs of Applications results}
Below are the proofs of the results presented in Section \ref{sec:applications}.\\

\begin{proof}[Proof of Lemma \ref{lem:Gibbs sampler fast mixing}]
For each $i,j\in V$, define
\begin{align*}
	R_{ij}
	&:=
	\begin{cases}
	0
	&\text{if } i= j,\\
	\displaystyle 
	\alpha(\beta)
	\max_{S\subseteq V : S\not\ni i,j}
	\left|
	1
	 -
	e^{\beta \Delta_j\Delta_if(S)}
	\right|
	&\text{if } i\neq j,
	\end{cases}
\end{align*}
with $\alpha(\beta) := \max_{i\in V} \max_{S\subseteq V\setminus\{i\}} e^{-\beta \Delta_if(S)}$. 
The first part of the Lemma follows immediately from Theorem \ref{thm:fastMCMC} and Remark \ref{rem:on fast mixing}, once we prove that $C\le R$, element-wise, where $C$ is the Dobrushin matrix defined in Lemma \ref{lem:Gibbs sampler}. In fact, condition \eqref{def:Dobrushin condition} yields
$$
	\| R \|_\infty = \max_{i\in V} \sum_{j\in V} R_{ij}
	\le \alpha(\beta) \max_{i\in V} \sum_{j\in V\setminus\{i\}} 
	\max_{S\subseteq V : S\not\ni i,j}
	\left|
	1
	 -
	e^{\beta \Delta_j\Delta_if(S)}
	\right|
	\le \gamma < 1,
$$
which corresponds to \eqref{fastmixinggeneral}. For each $i\in V$ we clearly have $C_{ii}=R_{ii}=0$. Henceforth, fix $i,j\in V$, $i \neq j$. As
\begin{align*}
	 \frac{|e^{\beta \Delta_i f(S)}\!-\!e^{\beta \Delta_i f(S\cup\{j\})}|}
	{(1\!+\!e^{\beta \Delta_i f(S)})(1\!+\!e^{\beta \Delta_i f(S\cup\{j\})})}
	\!&\le\!
	 |e^{-\beta \Delta_if(S\cup\{j\})} 
	 \!-\!
	 e^{-\beta \Delta_if(S)}|
	\!=\! e^{-\beta \Delta_if(S\cup\{j\})}
	\left|
	1
	 \!-\!
	 e^{\beta \Delta_j\Delta_if(S)} 
	\right| \!,
\end{align*}
taking the maximum over $S\subseteq V, S\not\ni i,j$, on both hand sides immediately yields $C_{ij}\le R_{ij}$. The proof of the Lemma is concluded once noticed that $f$ being submodular means that $\Delta_j\Delta_i(S) \le 0$ for each $i,j\in V$, $i\neq j$, $S\subseteq V$, $S\not\ni i,j$, and it implies $\Delta_if(V\setminus\{i\}) \le \Delta_i f(S)$ for each $i\in V$, $S\subseteq V$, $S\not\ni i$, so that $\alpha(\beta) = \max_{i\in V} e^{-\beta \Delta_if(V\setminus\{i\})}$.
\end{proof}

\begin{proof}[Proof of Corollary \ref{cor:facility}]
For each $i\in V$, $S\subseteq V$, $i\not\in S$, we have
$
	\Delta_i g(S)
	= \sum_{k=1}^m (L_{ki} - \max_{\ell \in S} L_{k\ell}) \vee 0,
$
and, clearly, $\Delta_i f(S) = \Delta_i g(S) - \lambda$.
The function $g$ is submodular as, for each $i,j\in V$, $i\neq j$, $S\subseteq V$, $i,j\not\in S$, we have
$$
	\Delta_i g(S\cup \{j\})
	=
	\sum_{k=1}^m (L_{ki} - \max_{\ell \in S\cup\{j\}} L_{k\ell}) \vee 0 
	\le \sum_{k=1}^m (L_{ki} - \max_{\ell \in S} L_{k\ell}) \vee 0
	=
	\Delta_i g(S),
$$
so also the function $f$ is submodular, as $\Delta_j\Delta_i f(S) = \Delta_j\Delta_i g(S) \le 0$.
Note that for each $i,j\in V$, $i\neq j$, $S\subseteq V$, $i,j\not\in S$, we can write
\begin{align*}
	\Delta_j\Delta_i f(S)
	&= \sum_{k=1}^m (L_{ki} - \max_{\ell \in S\cup\{j\}} L_{k\ell}) \vee 0
	- \sum_{k=1}^m (L_{ki} - \max_{\ell \in S} L_{k\ell}) \vee 0\\
	&= \sum_{k=1}^m \left(
	L_{ki} \wedge \max_{\ell \in S} L_{k\ell}
	- L_{ki} \wedge \max_{\ell \in S\cup\{j\}} L_{k\ell}
	\right)\\
	&= \sum_{k=1}^m \left(
	L_{ki} \wedge \max_{\ell \in S} L_{k\ell}
	- L_{ki} \wedge (L_{kj} \vee \max_{\ell \in S} L_{k\ell})
	\right)\\
	&= \sum_{k=1}^m
	\left\{
	(( \max_{\ell \in S} L_{k\ell} - L_{ki})\vee 0)
	1_{L_{ki} < L_{kj}}
	+
	(( \max_{\ell \in S} L_{k\ell} - L_{kj})\vee 0)
	1_{L_{ki} \ge L_{kj}}
	\right\},
\end{align*}
where we used, in order, the following two equalities holding for any real numbers $x,y,z$:
\begin{align*}
	(x-y)\vee 0 &= x - y\wedge x,
	\qquad
	x \wedge z - x \wedge (y\vee z) 
	&= ((z-x)\vee 0) \, 1_{x<y} + ((z-y)\vee 0) \,1_{x\ge y}.
\end{align*}
From the expression above it is clear that for each $i,j\in V$, $i \neq j$, we have
$
	\Delta_j\Delta_i f(S) \le \Delta_j\Delta_i f(S')
$
if $S\subseteq S'\subseteq V$, from which it follows that
$
	\min_{S\subseteq V : S\not\ni i,j} \Delta_j\Delta_if(S)
	= \Delta_j\Delta_if(\varnothing)
	= -\sum_{k=1}^m L_{ki} \wedge L_{kj}.
$
As
$
	\min_{i \in V}
	\Delta_if(V\setminus\{i\})
	\ge -\lambda,
$
the left hand side of \eqref{fastmixingsubmodular} is upper bounded by
\begin{align*}
	e^{\lambda\beta} \max_{i\in V} \sum_{j\in V\setminus\{i\}} 
	\left(
	1
	 -
	e^{-\beta \sum_{k=1}^m L_{ki} \wedge L_{kj}}
	\right).
\end{align*}
\end{proof}

\begin{proof}[Proof of Corollary \ref{cor:graphcut}]
It is east to check that for each $i,j\in V$, $i\neq j$, and $S\subseteq V$ so that $i,j\not\in S$ we have
$
	\Delta_if(S) = b \sum_{\ell \in V} L_{i\ell} - 2c \sum_{\ell\in S} L_{i\ell}
$
and
$
	\Delta_j\Delta_i f(S) = - 2 cL_{ij} \le 0,
$
from which it follows that $f$ is submodular. As
$
	\min_{i \in V} \Delta_if(V\setminus\{i\}) 
	= (b - 2c) \min_{i \in V} \sum_{\ell\in V\setminus\{i\}} L_{i\ell},
$
the left hand side of \eqref{fastmixingsubmodular} is upper bounded by
\begin{align*}
	e^{\beta (2c-b) \min_{i \in V} \sum_{\ell\in V\setminus\{i\}} L_{i\ell}}
	\max_{i\in V} \sum_{j\in V\setminus\{i\}} 
	\left(
	1 - e^{-2 c \beta L_{ij}}
	\right).
\end{align*}
\end{proof}

\begin{proof}[Proof of Corollary \ref{cor:ddp}]
For $i\not\in S\subseteq V$ we have
$$
	\det L_{S\cup \{i\}} = (\det L_{S})(L_{ii} - L_{i,S} L^{-1}_{S} L_{S,i}),
$$
where $L_{i,S} := (L_{ij})_{j\in S}\in\mathbb{R}^{1\times |S|}$ and $L_{S,i} := (L_{ji})_{j\in S}\in\mathbb{R}^{|S|\times 1}$.
It follows that
\begin{align*}
	f(S\cup\{i\}) 
	= \log\det L_{S\cup \{i\}}
	= f (S) + \log (L_{ii} - L_{i,S} L^{-1}_{S} L_{S,i}),
\end{align*}
and the marginal gain reads
$
	\Delta_if(S) = \log (L_{ii} - L_{i,S} L^{-1}_{S} L_{S,i})
	= \log \operatorname{Var}(X^i|X^S).
$
Analogously, if $i\not\in S$, $i\neq j$, we find
$
	\Delta_if(S\cup \{j\}) = \log (L_{ii} - L_{i,S\cup\{j\}} L^{-1}_{S\cup\{j\}} L_{S\cup\{j\},i}).
$
Recall that
$$
	L^{-1}_{S\cup\{j\}} =
	\left( 
	\begin{array}{cc}
	B & C \\
	C^T & \frac{1}{d}
	\end{array} 
	\right),
$$
with
$
	d := L_{jj} - L_{j,S} L^{-1}_{S} L_{S,j},
$
$
	B := L^{-1}_{S} + \frac{1}{d} L^{-1}_{S} L_{S,j} L_{j,S} L^{-1}_{S},
$
and
$
	C := -\frac{1}{d} L^{-1}_{S} L_{S,j}.
$
We get
\begin{align*}
	L_{i,S\cup\{j\}} L^{-1}_{S\cup\{j\}} L_{S\cup\{j\},i}
	&= 
	\left( L_{i,S}, L_{ij} \right)
	\left( 
	\begin{array}{cc}
	B & C \\
	C^T & \frac{1}{d}
	\end{array} 
	\right)
	\left( 
	\begin{array}{c}
	L_{S,i} \\
	L_{ji}
	\end{array} 
	\right)
	= L_{i,S} BL_{S,i} + 2 L_{ji}L_{i,S}C + \frac{L_{ij}^2}{d},
\end{align*}
where we used that $L_{ij}=L_{ji}$, and that
$
	 L_{ij}C^TL_{S,i} = L_{ij}(L_{S,i}^TC)^T = L_{ji}(L_{i,S}C)^T = L_{ji}L_{i,S}C.
$
Consequently,
\begin{align*}
	&L_{i,S\cup\{j\}} L^{-1}_{S\cup\{j\}} L_{S\cup\{j\},i}\\
	&\qquad\qquad= L_{i,S} \left( L^{-1}_{S} + \frac{1}{d} L^{-1}_{S} L_{S,j} L_{j,S} L^{-1}_{S} \right) L_{S,i} + 2 L_{ji}L_{i,S}\left(-\frac{1}{d} L^{-1}_{S} L_{S,j}\right) + \frac{L_{ij}^2}{d}\\
	&\qquad\qquad= L_{i,S} L^{-1}_{S} L_{S,i} 
	+ \frac{\left( L_{ij} - L_{i,S} L^{-1}_{S} L_{S,j} \right)^2}{L_{jj} - L_{j,S} L^{-1}_{S} L_{S,j}},
\end{align*}
where we used that
$
	L_{j,S} L^{-1}_{S} L_{S,i} 
	= (L_{j,S} L^{-1}_{S} L_{S,i})^T
	= L_{S,i}^T (L^{-1}_{S})^T L_{j,S}^T
	= L_{i,S} (L^{T}_{S})^{-1} L_{S,j}
	= L_{i,S} L^{-1}_{S} L_{S,j}.
$
Therefore, we find
$$
	\Delta_j\Delta_i f(S) = 
	\log \frac{L_{ii} - L_{i,S\cup\{j\}} L^{-1}_{S\cup\{j\}} L_{S\cup\{j\},i}}
	{L_{ii} - L_{i,S} L^{-1}_{S} L_{S,i}}
	= \log (1 - \rho(i,j|S)^2).
$$
As $0\le \rho(i,j|S)\le 1$ for each $i,j\in V$, $S\subseteq V$, it follows that $\Delta_j\Delta_i f(S)\le 0$ and $f$ is submodular.
From \eqref{def:condmutualinfo} we have
$
	\Delta_j\Delta_i f(S) = 
	-2 I(X^i;X^j|X^S),
$
and the left hand side of \eqref{fastmixingsubmodular} is equal to
\begin{align*}
	\max_{i\in V} \frac{1}{\operatorname{Var}(X^i|X^{V\setminus\{i\}})^{\beta}}
	\max_{i\in V} \sum_{j\in V\setminus\{i\}} 
	\max_{S\subseteq V : S\not\ni i,j}
	\left(1 - e^{-2 \beta I(X^i;X^j|X^S)}\right).
\end{align*}
The case $\beta=1$ follows immediately using \eqref{def:condmutualinfo}.
\end{proof}

\begin{proof}[Proof of Corollary \ref{cor:decomposable}]
First note that since $d\phi_A(x)/dx\geq c$ for each $x\in[0,n]$ it follows that the discrete differences are all bounded, i.e., $\phi_A(x+1) - \phi_A(x)\geq c$ for each $x\in \{0,\ldots,n-1\}$ and $A\in\mathcal{S}$. Similarly, since $d^2\phi_A(x)/dx^2\geq c'$ for each $x\in[0,n]$ we have $\phi_A(x+2) - 2 \phi_A(x+1)+ \phi_A(x)\geq c'$ for each $x\in \{0,\ldots,n-2\}$ and $A\in\mathcal{}S$. Now, for each $i\in V$, $S\subseteq V$, $i\not\in S$ we have
$$
	\Delta_i f(S) = \sum_{A\in\mathcal{S}: A\ni i} 
	\left\{\phi_A(|A\cap (S\cup\{i\})|) - \phi_A(|A\cap S|)\right\},
$$
and each $i,j\in V$, $i\neq j$, $S\subseteq V$, $i,j\not\in S$ we have
\begin{align*}
	&\Delta_j\Delta_i f(S)\\
	&=\!\sum_{A\in\mathcal{S}: A\ni i,j} \!\!\!\!\!
	\left\{
	\phi_A(|A\cap (S\cup\{i,j\})|) \!-\! \phi_A(|A\cap (S\cup\{j\})|)
	\!-\! \phi_A(|A\cap (S\cup\{i\})|) \!+\! \phi_A(|A\cap S|)
	\right\}\\
	&=\!\sum_{A\in\mathcal{S}: A\ni i,j} 
	\left\{
	\phi_A(|A\cap S|+2) - 2 \phi_A(|A\cap S|+1) + \phi_A(|A\cap S|)
	\right\} \le 0,
\end{align*}
where the inequality comes from the concavity of each $\phi_A$, as
$\phi_A(x+2) - 2\phi_A(x+1) + \phi_A(x) \le 0$ for $x\ge 0$. As $\Delta_j\Delta_i f(S) \le 0$, $f$ is submodular. In particular, note that if $i,j$ are such that there is no $A\in\mathcal{S}$ that satisfies $A\ni i,j$, then the previous expression yields
$
	\Delta_j\Delta_i f(S) = 0.
$
As for each $i\in V$ we have
$
	\Delta_i f(V\setminus\{i\}) = \sum_{A\in\mathcal{S}: A\ni i} 
	\left\{\phi_A(|A|) - \phi_A(|A|-1)\right\}
	\ge c |\{A\in \mathcal{S}: A\ni i\}|,
$
and for each $i,j\in V$, $i\neq j$, $S\subseteq V$, $i,j\not\in S$ we have
$
	\Delta_j\Delta_if(S) \ge c',
$
then the left hand side of \eqref{fastmixingsubmodular} is upper bounded by
\begin{align*}
	&(1-e^{c'\beta})
	e^{-c \beta \min_{i\in V}  |\{A\in \mathcal{S}: A\ni i\}|} 
	\max_{i\in V} |\{j\in V\setminus\{i\}: i,j\in A \text{ for some }A\in\mathcal{S}\}|\\
	&\qquad\qquad\le (1-e^{c'\beta})
	e^{-c \beta \min_{i\in V}  |\{A\in \mathcal{S}: A\ni i\}|} 
	\max_{i\in V} \left|\bigcup_{A\in\mathcal{S}:A\ni i} A\right|.
\end{align*}
\end{proof}

\end{document}